\documentclass{llncs}

\usepackage[T1]{fontenc} 
\usepackage{tgtermes}
\usepackage{amssymb}
\usepackage{amsmath}
\usepackage{graphicx}
\usepackage{esint}
\usepackage{enumitem}
\usepackage{algorithm}
\usepackage{hyperref}
\usepackage{marvosym}
\usepackage{mathtools}
\usepackage{xspace} 
\usepackage{color}
\usepackage{comment}

\usepackage[noend]{algpseudocode}
\usepackage[section]{placeins}
\usepackage{hyperref}

\algrenewcommand\algorithmicindent{1.0em}

\newcommand{\GLC}{\ensuremath{\mathrm{GLC}}\xspace}
\newcommand{\PRM}{\ensuremath{\mathrm{PRM}}\xspace}
\newcommand{\PRMs}{\ensuremath{\mathrm{PRM}^*}\xspace}
\newcommand{\RRT}{\ensuremath{\mathrm{RRT}}\xspace}
\newcommand{\RRTs}{\ensuremath{\mathrm{RRT}^*}\xspace}
\newcommand{\SST}{\ensuremath{\mathrm{SST}}\xspace}
\newcommand{\EST}{\ensuremath{\mathrm{EST}}\xspace}
\newcommand{\NULL}{\ensuremath{\mathtt{NULL}}\xspace}

\newcommand{\U}{\ensuremath{\mathcal{U}}\xspace}

\newcommand{\Ugoal}{\ensuremath{\mathcal{U}_{\mathrm{goal}}}\xspace}

\newcommand{\Xxo}{\ensuremath{\mathcal{X}_{x_{0}}}\xspace}

\begin{document}
\title{A Generalized Label Correcting Method for \\ Optimal Kinodynamic Motion Planning }

\author{Brian Paden and Emilio Frazzoli}
\institute{Massachusetts Institute of Technology \\ 
\email{bapaden@mit.edu}, $\quad$ \email{frazzoli@mit.edu}}
\maketitle
\begin{abstract}
An optimal kinodynamic motion planning algorithm is presented and  described as a generalized label correcting (\GLC) method.
%
%
Advantages of the technique include a simple implementation, convergence to the optimal cost with increasing resolution, and no requirement for a point-to-point local planning subroutine.
The principal contribution of this paper is the construction and analysis of the \GLC conditions which are the basis of the proposed algorithm.
%
%
%

\end{abstract}

\section{Introduction}
Motion planning is a challenging and fundamental problem in robotics research with numerous algorithms fine-tuned to variations of the problem.
Among the most popular is the \PRM~\cite{PRM}, an algorithm for planning in relatively high dimensions, but requiring a point-to-point local planning subroutine or steering function to connect pairs of randomly sampled states.
A variation, \PRMs~\cite{karaman2011sampling}, converges to an optimal solution with respect to some objective provided the steering function is optimal.
%

%
The \EST~\cite{EST_Journal} and \RRT~\cite{RRT_Journal} algorithms were developed to address suboptimal planning with differential constraints where a steering function is unavailable.
The steering function in these algorithms is replaced by a subroutine which forward integrates the system dynamics with a selected control input.
Optimal planning under differential constraints can be addressed by the kinodynamic variant of the \RRTs algorithm~\cite{karaman2010optimal}, but again requires an optimal steering function like \PRMs. 
The effectiveness of \RRTs has motivated several general approaches for optimal steering discussed in~\cite{perez2012lqr,xie2015toward}.
However, the computation of these solutions can slow down the iteration time of \RRTs considerably.
%

%
%
%

%
This paper addresses optimal planning under differential constraints without the use of a steering function. 
The method is based on a discrete approximation of the problem as a graph together with a generalization of label correcting techniques such as Dijkstra's algorithm~\cite{dijkstra1959note}.
Label correcting algorithms compare the relative cost of paths in a graph terminating at the same vertex and discard paths with non-minimal cost.
This is effective when there are multiple paths reaching each vertex from the root vertex.
Without a steering function it is difficult to construct such a graph approximating motion planning solutions since multiple trajectories terminating at specified states must be generated.
The intuitive generalization is to compare the cost of paths related to trajectories terminating "close enough" to one another.
This concept first appeared in~\cite{dolgov2010path} as the Hybrid $\rm A^*$ algorithm, but without conditions for resolution completeness.
More recently, the \SST algorithm~\cite{Li2016Asymptotically-} provided a more principled approximation of the problem and an algorithm converging asymptotically to an approximately optimal solution. 
%
 
%
%
%
%
%
%

%
%
%
This paper refines this approach further with an algorithm producing approximate solutions in finite-time. 
The key difference allowing the algorithm to terminate in finite time is a slightly more conservative comparison between paths terminating in the same region of the state space. 
This comparison is described in Section \ref{sec:GLC-Methods}.
Section \ref{sec:Numerical-Examples} provides examples illustrating convergence to optimal cost solutions with increasing resolution. 
The analysis of the algorithm is addressed in Section \ref{sec:Justification} and is the principal contribution of the paper.
\section{\label{sec:GLC-Methods}GLC Methods}

\paragraph{Label correcting methods:}
Shortest path algorithms are methods of searching over paths of a graph for a minimum-cost path between an origin or root vertex to a set of goal vertices.
In a conventional label correcting method, the algorithm maintains a best known path terminating at each vertex of the graph. This path \textit{labels} that vertex. 
At a particular iteration, if a path under consideration does not have lower cost than the path labeling the terminal vertex, the path under consideration is discarded. 
As a consequence, the subtree of paths originating from the discarded path will not be evaluated.

%

\paragraph{Generalizing the notion of a label:}
Observe that the label of a vertex in conventional label correcting algorithms is in fact a label for the paths terminating at that vertex.
Then each vertex identifies an equivalence class of paths.
Paths within each equivalence class are ordered by their cost, and the efficiency of label correcting methods comes from narrowing the search to minimum cost paths in their associated equivalence class. 
The generalization is to identify paths associated to trajectories terminating in the same region of the state space instead of the same state.
However, this generalization prevents a direct comparison of cost between two related paths. 
Instead, the difference in cost must exceed a threshold  described by the \GLC conditions introduced in Section \ref{sec:glc_def}.

\subsection{Problem Formulation}

Consider a system whose configuration and relevant quantities are described by a state in $\mathbb{R}^{n}$. 
The decision variable of the problem is an input signal or continuous history of actions $u$ from a \textit{signal space} $\mathcal{U}$ affecting a state trajectory $x$ in a \textit{trajectory space} $\mathcal{X}_{x_0}$. The signal space is constructed from a set of admissible inputs $\Omega\subset\mathbb{R}^{m}$ bounded by $u_{max}$. The input signal space is defined
\begin{equation}
\mathcal{U}\coloneqq\bigcup_{\tau>0}\left\{ u\in L_{1}([0,\tau]):\: u(t)\in\Omega\:\,\forall t\in[0,\tau]\right\}. \label{eq:signal_space}
\end{equation}

The signal and trajectory are related through a model of the system dynamics described by a differential
constraint,
\begin{equation}
\frac{d}{dt}x(t)=f(x(t),u(t)), \quad x(0)=x_0\label{eq:dynamics}.
\end{equation}
A \textit{system map} $\varphi_{x_0}:\mathcal{U}\rightarrow\mathcal{X}_{x_0}$ is defined to relate signals to associated trajectories (the solution to (\ref{eq:dynamics})~\cite[cf. pg. 42]{coddington1955theory}) with domain equal to the domain of the input signal. 
The initial state $x_0$ parametrizes the map and trajectory space. 
To simplify notation, $\tau(x)$ for a function $x$ with domain $[t_1,t_2]$ denotes the maximum of the domain, $t_2$.   

In addition to the differential constraint, feasible trajectories for a particular problem must satisfy point-wise constraints defined by a subset $X_\mathrm{free}$ of $\mathbb{R}^n$ and a specified initial state $x_\mathrm{ic}$. 
%
The subset of feasible trajectories $\mathcal{X}_\mathrm{feas}$ are defined 
\begin{equation}
\mathcal{X}_\mathrm{feas}\coloneqq \left\{ x\in \mathcal{X}_{x_\mathrm{ic}}:\,\, x(t)\in X_\mathrm{free}\,\, \forall t\in [0,\tau (x)] \right\}.
\end{equation}
Similarly, the subset $X_\mathrm{goal}$ of $\mathbb{R}^n$  is used to encode a terminal constraint. The subset of $\mathcal{X}_\mathrm{feas}$ consisting of trajectories $x$ with $x(\tau(x))\in X_\mathrm{goal}$ defines $\mathcal{X}_\mathrm{goal}$.
%
%
%

%
The decision variable for the problem is the input signal $u$ which must be chosen such that the trajectory $\varphi_{x_\mathrm{ic}}(u)$ is in $\mathcal{X}_\mathrm{goal}$.
Naturally, input signals mapping to trajectories in $\mathcal{X}_\mathrm{feas/goal}$ are defined by the inverse relation  $\mathcal{U}_\mathrm{feas/goal} \coloneqq \varphi_{x_\mathrm{ic}}^{-1}(\mathcal{X}_\mathrm{feas/goal})$. 

A general cost functional which integrates a running-cost $g$ of the state and input at each instant along a trajectory is used to compare the merit of one input signal over another.
Restricted to solutions of (\ref{eq:dynamics}), the cost functional depends only on the control and intitial state, 
\begin{equation}
J_{x_{0}}(u)=\int_{[0,\tau(u)]}g([\varphi_{x_{0}}(u)](t),u(t))\, d\mu(t).\label{eq:real_cost}
\end{equation}
The notation $[\varphi_{x_{0}}(u)](t)$ denotes the evaluation of the trajectory satisfying (\ref{eq:dynamics}) with the input signal $u$ and initial state $x_{0}$ at time $t$. 

The domain of $J$ is extended with the object \NULL such that $J_{x_0}(\NULL)=\infty$ for all $x_0\in\mathbb{R}^{n}$.
Further, since (\ref{eq:real_cost}) may not admit a minimum, we address the following relaxed optimal kinodynamic motion planning problem:
\begin{problem}\label{Problem}
Find a sequence $\{u_{R}\}\subset\mathcal{U}_\mathrm{goal}\cup \NULL$ such that 
\begin{equation}
\lim_{R\rightarrow\infty}J_{x_\mathrm{ic}}(u_R)=\inf_{u\in\mathcal{U}_\mathrm{goal}}J_{x_\mathrm{ic}}(u)\coloneqq c^{*}.\label{eq:meaningful_problem}
\end{equation}
With the convention that $\inf_{u\in\emptyset}J_{x_\mathrm{ic}}(u)=\infty$, a solution sequence exists so the problem is well-posed. An algorithm parameterized by a resolution $R \in \mathbb{N}$ whose output for each $R$ forms a sequence solving this problem will be called \textit{resolution complete}.
\end{problem}

\paragraph{Assumptions.} The problem data are assumed to satisfy  the following:
\begin{enumerate}[label=\textbf{A-\arabic*},itemindent=0.25cm] 
\item The sets $X_\mathrm{free}$ and $X_\mathrm{goal}$ are open with respect to the standard topology on $\mathbb{R}^n$.
\item There are known constants $L_{f}\geq0$ and $M\geq0$ such that $\left\Vert f(x_{1},u)-f(x_{2},u)\right\Vert _{2}\leq L_{f}\left\Vert \left(x_{1}-x_{2}\right)\right\Vert _{2}$, and $\left\Vert f(x_{1},u)\right\Vert _{2}\leq M$ for all $x_{1},x_{2}\in X_\mathrm{free}$ and $u\in\Omega$.
\item  There is a known constant $L_{g}\geq0$ such that $\left\Vert g(x_{1},u_{1})-g(x_{2},u_{2})\right\Vert _{2}\leq$\\$ L_{g}\left\Vert \left(x_{1}-x_{2},u_{1}-u_{2}\right)\right\Vert _{2}$ for all $x_{1},x_{2}\in X_\mathrm{free}$ and $u_{1},u_{2}\in\Omega$.
\item  $ J_{x_{ic}}(u)>0 $ for all $u\in\mathcal{U}$.
\end{enumerate} 

\remark{
%
%

%
%
We do not require the reachable set to have a nonempty interior as in the kinodynamic variant of \RRTs~\cite{karaman2010optimal} and \SST~\cite{Li2016Asymptotically-}.
}

\subsection{Approximation of $\mathcal{U}$}

The signal space $\mathcal{U}$ is approximated by a finite subset  $\mathcal{U}_{R}$ where $R\in\mathcal{\mathbb{N}}$ is a resolution parameter. 
To construct $\mathcal{U}_{R}$ it is assumed that we have access to a family of finite subsets $\Omega_{R}\subset\Omega$ of the input space, such that the dispersion
in $\Omega$ converges to zero as $R\rightarrow \infty$. 
A family of such subsets exists and is often easily obtained with regular grids or random sampling for a given $\Omega$.   

The approximated signal space $\mathcal{U}_R$ consists of piecewise constant signals on time intervals $\left[\frac{c\cdot(i-1)}{R},\frac{c\cdot i}{R}\right)$ for $c>0$, $i\in\{1,...,d\}$, and all values $d$ less than a user specified function $h(R)$ (take the last interval to be closed). The constant values of the signal are inputs from $\Omega_R$.   

The function $h:\mathbb{N}\rightarrow\mathbb{N}$ defines a horizon limit and can be \textit{any} function satisfying
\begin{equation}
\lim_{R\rightarrow\infty}R/h(R)=0.\label{eq:h_constraint}
\end{equation}
For example, $h(R)=R\log(R)$ is acceptable. 
This ensures that the horizon limit is unbounded in $R$ so that any finite time domain can be approximated for sufficiently large $R$.
A \textit{parent} of an input signal $w\in\mathcal{U}_{R}$ with domain $\left[0,\frac{c\cdot i}{R}\right]$ is defined as the input signal $u\in\mathcal{U}_{R}$ with domain $\left[0,\frac{c\cdot(i-1)}{R}\right]$ such that $w(t)=u(t)$ for all $t\in\left[0,\frac{c\cdot(i-1)}{R}\right)$.
In this case, $w$ is a \textit{child} of $u$. Two signals are \emph{siblings} if they have the same parent.
A tree (graph) is defined with $\mathcal{U}_{R}$ as the vertex set, and edges defined by ordered pairs of signals $(u,w)$ such that $u$ is the parent of $w$.
To serve as the root of the tree, $\mathcal{U}_R$ is augmented with the special input signal $Id_{\mathcal{U}}$ defined such that $J_{x_0}(Id_{\mathcal{U}})=0$ and $[\varphi_{x_0}(Id_\mathcal{U})](0)=x_0$. $Id_{\mathcal{U}}$ has no parent, but is the parent of signals with domain $[0,c/R]$. 

The signal $w$ is an \textit{ancestor} of $u$ if $\tau(w)\leq\tau(u)$ and $w(t)=u(t)$ for all $t\in[0,\tau(w))$.
In this case $u$ is a \textit{descendant} of $w$. The \textit{depth} of a control in $\mathcal{U}_{R}$ is the number of ancestors of that control.

\subsection{Partitioning $X_\mathrm{free}$ and the GLC Conditions}
\label{sec:glc_def}
%
%
A partition of $X_\mathrm{free}$ is used to define comparable signals. We say the partition has radius $r$ if the partition elements are each contained in a neighborhood of radius $r$.
Like the discretization of $\mathcal{U}$, the radius of the partition is parametrized by the resolution $R$. 

For brevity we only consider hypercube partitions.
A user specified function $\eta:\mathbb{N}\rightarrow\mathbb{R}_{>0}$ controls the side length of the hypercube. 
For states $p_1,p_2\in\mathbb{R}^n$ we write $p_1\overset{R}{\sim}p_2$ if $\left\lfloor \eta(R)p_1\right\rfloor =\left\lfloor \eta(R)p_2\right\rfloor$, where $\left\lfloor \cdot\right\rfloor $ is the coordinate-wise floor map (e.g. $\left\lfloor (2.9,3.2) \right\rfloor = (2,3)$). 
Then the equivalence classes of the $\overset{R}\sim$ relation define a simple hypercube partition of radius $\sqrt{n}/\eta(R)$.    
We extend this relation to control inputs by comparing the terminal state of the resulting trajectory. 
For $u_1,u_2 \in \mathcal{U}_R$ we write $u_{1}\overset{\mathcal{U}_R}{\sim}u_{2}$ if the resulting trajectories terminate in the same hypercube. That is,
\begin{equation}\label{eq:equiv_signal}
u_{1}\overset{\mathcal{U}_R}{\sim}u_{2} \Leftrightarrow [\varphi_{x_\mathrm{ic}}(u_{1})](\tau(u_{1})) \overset{R}{\sim}  [\varphi_{x_\mathrm{ic}}(u_{2})](\tau(u_{2})) .
\end{equation}
Figure \ref{fig:intuition} illustrates the intuition behind this equivalence relation.
%
%
%

%
To compare signals we write $u_{1}\prec_{R}u_{2}$ if the \GLC conditions are satisfied:
\begin{enumerate}[label=\textbf{GLC-\arabic*},itemindent=0.75cm] 
\item \label{glc1} $u_{1}\overset{\mathcal{U}_R}\sim u_{2}$,
\item \label{glc2} $\tau(u_{1})\leq\tau(u_{2}),$
\item \label{glc3} $J_{x_\mathrm{ic}}(u_{1})+\frac{\sqrt{n}}{\eta(R)}\frac{L_{g}}{L_{f}}\left(e^{\frac{L_{f}h(R)}{R}}-1\right) \leq J_{x_\mathrm{ic}}(u_{2}).$
\end{enumerate}
A signal $u_1$ is called \textit{minimal} if there is no $u_2\in \mathcal{U}_R$ such that $u_{2}\prec_{R}u_{1}$. Such a signal can be thought of as being a good candidate for later expansion during the search. Otherwise, it can be discarded.
In order for the \GLC method to be a resolution complete algorithm, the scaling parameter $\eta$ must satisfy  
\begin{equation}
\lim_{R\rightarrow\infty}\frac{R}{L_f\eta(R)}\left(e^{\frac{L_{f}h(R)}{R}}-1\right)=0.\label{eq:partition_scaling}
\end{equation}
%
%
%

%
Observe that (\ref{eq:partition_scaling}) implies the cost threshold in \ref{glc3} is in $O(1/R)$ and converges to zero.
A condition yielding the same theoretical results, but asymptotically requiring more signals to be evaluated would be to replace the threshold with an arbitrarily small positive constant. 
Additionally, (\ref{eq:partition_scaling}) and \ref{glc3} simplify in some cases.
For kinematic problems $L_{f}=0$.
Taking the limit $L_{f}\rightarrow 0$, the constraint (\ref{eq:partition_scaling}) becomes $h(R)/\eta(R)\rightarrow 0$.
The second special case is minimum-time problems where $g(x,u)=1$ in (\ref{eq:real_cost}) so that $L_{g}=0$. 
Then \ref{glc3} simplifies to $J_{x_\mathrm{ic}}(u_{1}) \leq J_{x_\mathrm{ic}}(u_{2})$.

\begin{figure}[htb]
\centering{}\includegraphics[width=1\textwidth]{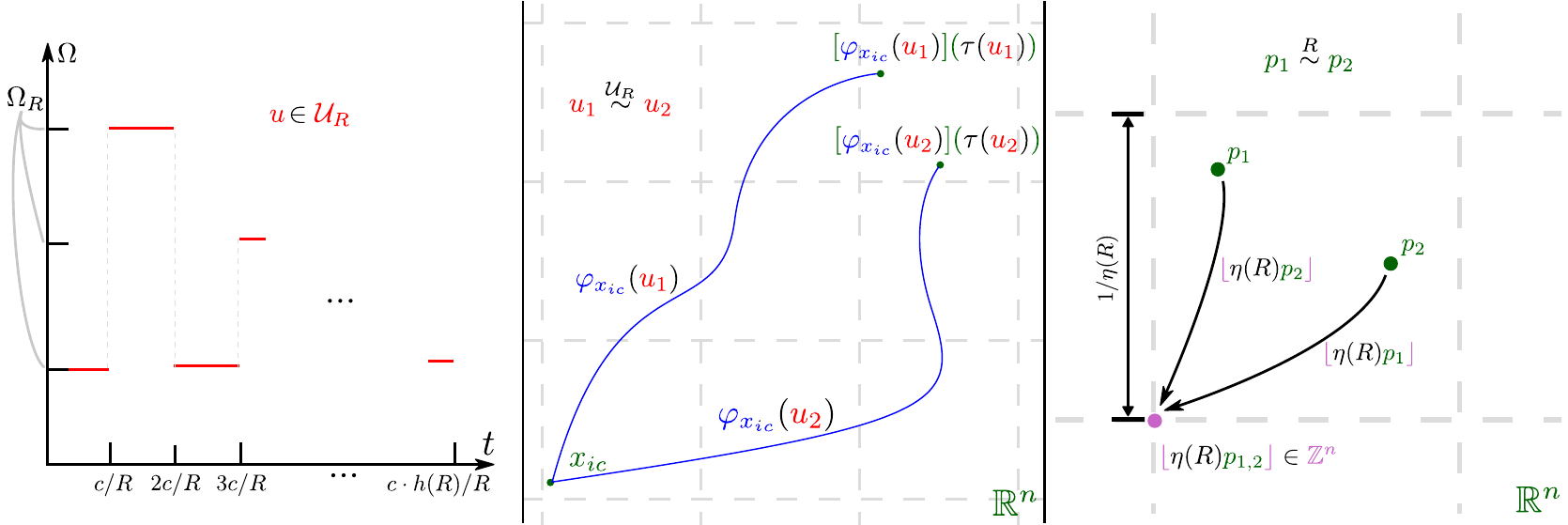}
\caption{\label{fig:intuition} Color coded depictions of: (left) a signal from $\mathcal{U}_R$ with signals represented in red, (middle) the mapping into the trajectory space $\mathcal{X}_{x_\mathrm{ic}}$ in blue for two equivalent signals $u_1 \overset{\mathcal{U}_R}{\sim}u_2$, and (right) the mapping from terminal states in $\mathbb{R}^n$ shown in green into $\mathbb{Z}^n$ by the floor map.}
\end{figure}

\subsection{Algorithm Description} 

Pseudocode for the \GLC method is described in Algorithm \ref{Alg} below. A set $Q$ serves as a priority queue of candidate signals. A set $\Sigma$ contains signals representing labels of $\overset{\mathcal{U}_R}{\sim}$ equivalence classes.

The method $\mathtt{expand}(u)$ returns the set of all children of $u$. 
The method $\mathtt{pop}(Q)$ deletes from Q, and returns an input signal $\hat{u}$ such that 
\begin{equation}
\hat{u}\in\underset{u\in Q}{{\rm argmin}}\left\{ J_{x_\mathrm{ic}}(u)\right\}. \label{eq:queue}
\end{equation} 
The addition of an admissible heuristic~\cite{hart1968formal} in (\ref{eq:queue}) can be used to guide the search without affecting the solution accuracy.  

The method $\mathtt{find}(u,\Sigma)$ returns $w\in\Sigma$ such that $u\overset{R}{\sim}w$ or \NULL if no such $w$ is present in $\Sigma$.
Problem specific collision and goal checking subroutines are used to evaluate $u\in\mathcal{U}_\mathrm{feas}$ and $u\in\mathcal{U}_\mathrm{goal}$.
The method $\mathtt{depth}(u)$ returns the number of ancestors of $u$. 
\begin{algorithm} 
\begin{algorithmic}[1]
\State $Q\leftarrow \{Id_\mathcal{U}\},\,\Sigma \gets \emptyset,\,S \gets \emptyset$  
\While {$Q\neq \emptyset$}        
\State $u \gets \mathtt{pop}(Q)$   
\If{$u \in \mathcal{U}_\mathrm{goal} $}         	  
\State \Return $(J_{x_{ic}}(u),u)$ 	
\EndIf       
\State $S \gets \mathtt{expand}(u)$   
\For{$w \in S$} 	
\State $z \gets \mathtt{find}(w,\Sigma)$ 	  
\If{$(w \notin \mathcal{U}_{feas.} \vee (z \prec_R w) \vee \mathtt{depth}(w) \geq h(R))$}  	    
\State $S \gets S\setminus \{w\}$  	  
\ElsIf{$J_{x_\mathrm{ic}}(w)<J_{x_\mathrm{ic}}(z)$} 	  
\State $\Sigma \gets (\Sigma \setminus \{z\}) \cup \{w\}$ 	  
\EndIf 
\EndFor       
\State $Q \gets Q \cup S$ 
\EndWhile  
\State 
\Return $(\infty,\NULL)$ 
\end{algorithmic} \caption{\label{Alg} Generalized Label Correcting (GLC) Method} 
\end{algorithm}
The algorithm begins by adding the root $Id_\mathcal{U}$ to the queue (line 1), and then enters a loop which recursively removes and expands the top of the queue (line 3) adding children to $S$ (line 4). 
If the queue is empty the algorithm terminates (line 2) returning \NULL (line 14). 
Each signal removed from $Q$ (line 5) is checked for membership in $\mathcal{U}_\mathrm{goal}$ in which case the algorithm terminates returning a feasible solution with approximately the optimal cost. 
Otherwise, the signals are checked for infeasibility or suboptimality by the \GLC conditions (line 9). 
Next, a relabeling condition for the associated equivalence classes (i.e. grid cells) of remaining signals is checked (line 11). 
Finally, remaining signals in $S$ are added to the queue (line 13).  
%

%
The main result of this paper, justified in Section \ref{sec:Justification}, is the following:
\begin{theorem}\label{thm:informal}
	Let $w_R$ be the signal returned by the \GLC method for resolution $R$. Then $\lim_{R\rightarrow \infty} J_{x_{ic}}(w_R)=c^*$. That is, the \GLC method is a resolution complete algorithm for the optimal kinodynamic motion planning problem.
\end{theorem}
This conclusion is independent of the order in which children of the current node are examined in line 7. 

\section{\label{sec:Numerical-Examples}Numerical Experiments}

The \GLC method (Algorithm \ref{Alg}) was tested on five problems and compared, when applicable, to the implementation of \RRTs from~\cite{rrt_implementation} and \SST from~\cite{BBekris2015}. 
The goal is to examine the performance of the \GLC method on a wide variety of problems. 
The examples include under-actuated nonlinear systems, multiple cost objectives, and environments with/without obstacles.
Note that adding obstacles effectively speeds up the \GLC method since it reduces the size of the search tree.
Another focus of the examples is on real-time application.
In each example the running time for \GLC method to produce a (visually) acceptable trajectory is comparable to the execution time.
Of course this will vary with problem data and computing hardware.    

\paragraph{Implementation Details:}

The \GLC method was implemented in C++ and run with a 3.70GHz Intel Xeon CPU. 
The set $Q$ was implemented with an STL priority queue so that the $\mathtt{pop}(Q)$ method and insertion operations have logarithmic complexity. 
The set $\Sigma$ was implemented with an STL set which uses a binary search tree so that $\mathtt{find}(w,\Sigma)$ has logarithmic complexity as well. 

Sets $X_\mathrm{free}$ and $X_\mathrm{goal}$ are described by algebraic inequalities.
The approximation of the input space $\Omega_{R}$ is constructed by uniform deterministic sampling of $R^{m}$ controls from $\Omega$.
Recall $m$ is the dimension of the input space.

Evaluation of $u\in\mathcal{U}_\mathrm{feas}$ and $u\in\mathcal{U}_\mathrm{goal}$ is approximated by first numerically approximating $\varphi_{x_\mathrm{ic}}(u)$ with Euler integration (except for \RRTs which uniformly samples along the local planning solution). 
The number of time-steps is given by $N=\left\lceil \tau(u)/\Delta\right\rceil $ with duration $\tau(u)/N$. Maximum time-steps $\Delta$ are 0.005 for the first problem, 0.1 for the second through fourth problem, and 0.02 for the last problem. Feasibility is then approximated by collision checking at each time-step along the trajectory.

\paragraph*{Shortest path problem: }

A shortest path problem can be represented by the dynamics $f(x,u)=u$ with the running-cost $g(x,u)=1$ and input space $\Omega=\{u\in\mathbb{R}^{2}:\,\Vert u\Vert_{2}=1\}.$
The free space and goal are illustrated in Figure \ref{fig:bench}.
The parameters for the \GLC method are $c=10$, $\eta(R)=R^2/300$, and $h(R)=100R\log(R)$ with resolutions $R\in \{20,25,...,200\}$.
The exact solution to this problem is known so we can compare the relative convergence rates of the \GLC method, \SST, and \RRTs. 

\paragraph*{Torque limited pendulum swing-up: }

The system dynamics are $f(\theta,\omega,u)=(\omega,u-\sin(\theta))^{T}$
with the running-cost $g(\theta,\omega,u)=1$ and input space  $\Omega=[-0.2,0.2]$.
The free space is modeled as $X_{free}=\mathbb{R}^2$.
The initial state is $x_{ic}=(0,0)$ and the goal region is $X_{goal}=\{x\in \mathbb{R}^2:\,\Vert(\theta \pm \pi,\omega \Vert_2 \leq 0.1\}$. 
The parameters for the \GLC method are $c=6$, $\eta(R)=R^{2.5}/16$, and $h(R)=100R\log(R)$ with resolutions $R\in \{4,5,6,7,8\}$.
The optimal solution is unknown so only the running time vs. cost can be plotted. 
Without a local planning solution \RRTs is not applicable. 
The same is true for the remaining examples.

\paragraph*{Torque limited acrobot swing-up: }

The acrobot is a double link pendulum actuated at the middle joint. 
The expression for the four dimensional system dynamics are cumbersome to describe and we refer to~\cite{spong1995swing} for the details. 
The model parameters, free space, and goal region are identical to those in \cite{BBekris2015} with the exception that the radius of the goal region is reduced to 0.5 from 2.0.
The running-cost is $g(x,u)=1$ and the input space is $\Omega=[-4.0,4.0]$. 
The parameters for the \GLC method are $c=6$, $\eta(R)=R^{2}/16$, and $h(R)=100R\log(R)$ with resolutions $R\in \{5,6,...,10\}$.

\paragraph*{Acceleration limited 3D point robot: }

To emulates the mobility of an agile aerial vehicle (e.g. a quadrotor with high bandwidth attitude control), the system dynamics are $f(x,v,u)=(v,5.0u-0.1v\Vert v\Vert_{2})^{T}$ where $x$, $v$, and $u$ are each elements of $\mathbb{R}^{3}$; there are a total of six states. 
The quadratic dissipative force antiparallel to the velocity $v$ models aerodynamic drag during high speed flight.
The running-cost is $g(x,v,u)=1$ and the input space is $\Omega=\{u\in\mathbb{R}^{3}:\Vert u\Vert_{2}\leq1\}$.
The free space and goal are illustrated in Figure \ref{fig:bench}. 
The sphere in the rightmost room illustrates the goal configuration. The terminal velocity is left free. The velocity is initially zero, and the cylinder indicates the starting configuration.
The parameters for the \GLC method are $c=10$, $\eta(R)=R^{3/2}/64$, and $h(R)=100R\log(R)$ with resolutions $R\in \{8,9,...,14\}$.
A guided search is also considered with heuristic given by the distance to the goal divided by the maximum speed $v$ of the robot (A maximum speed of $\sqrt{{50}}$ can be determined from the dynamics and input constraints). 

\paragraph*{Nonholonomic wheeled robot:}

The system dynamics emulating the mobility of a wheeled robot are $f(x,y,\theta,u)=$  $(\cos(\theta),\sin(\theta),u)^{T}$. 
The running cost is $g(x,y,\theta,u)=1+2u^{2}$, with input space $\Omega=[-1,1]$.
The parameters for the \GLC method are $c=10$, $\eta(R)=15R^{5/\pi}$, and $h(R)=5R\log(R)$ with resolutions $R\in \{4,5,...,9\}$.
Note the quadratic penalty on angular rate which is relevant to rider comfort specifications in driverless vehicle applications.

\begin{figure}
	\centering{}\includegraphics[width=1\textwidth]{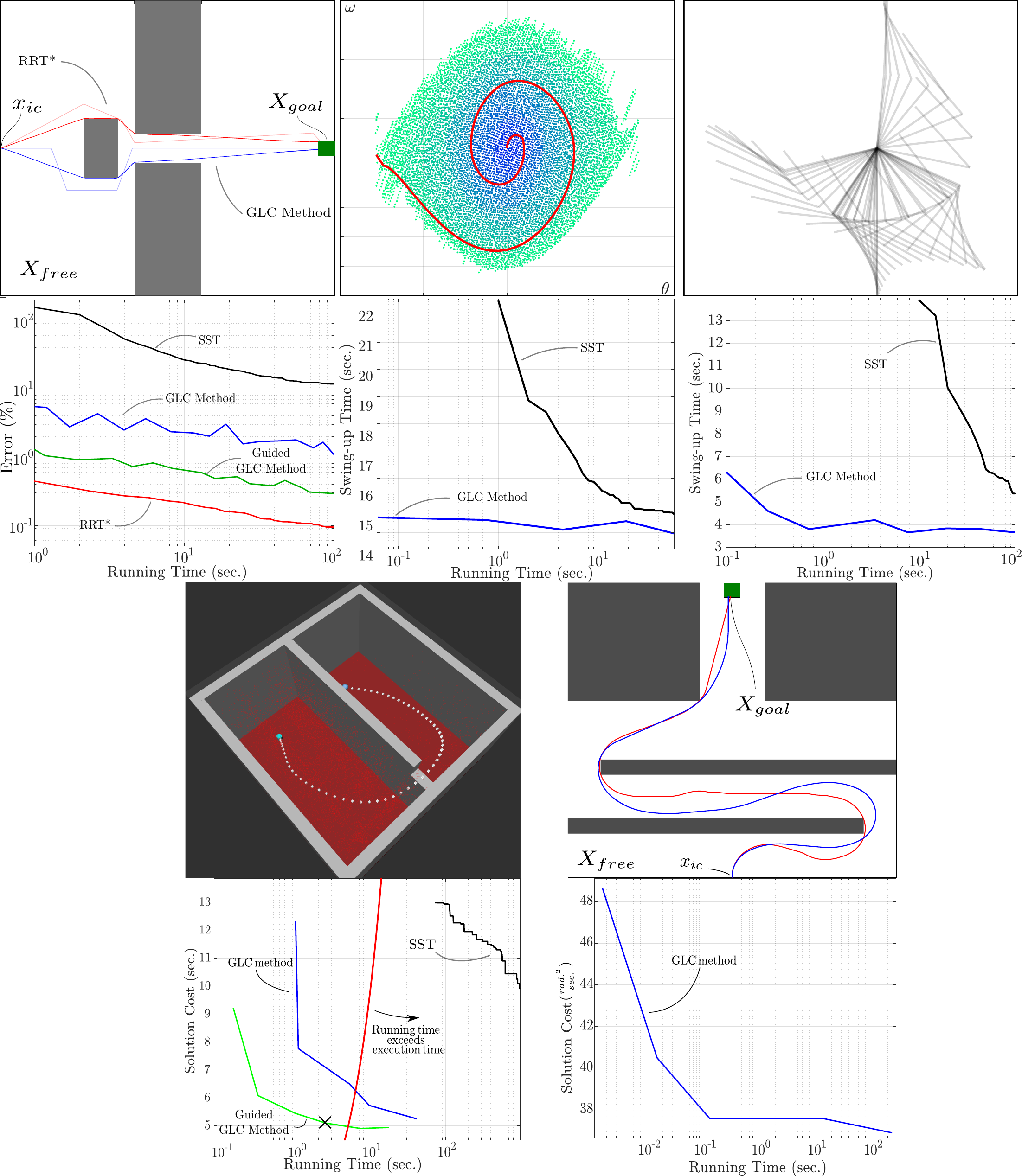}\caption{\label{fig:bench} Running times are based on 10 trial average for randomized planners and only reported if a solution was found at a particular time in all 10 trials. {\bf Shortest path example (top-left)}; the opaque and solid paths illustrate first and last solutions obtained respectively. {\bf Pendulum example (top-middle)}; the graphic illustrates the \GLC solution for $R=7$ with the colored markers indicating the cost of grid labels. {\bf Acrobot example (top-right)}; The state space has four dimensions. Only the configuration is illustrated. {\bf Acceleration limited point robot (bottom-left)}; The graphic shows the best \GLC solution with a running time less than the execution time. The $\times$ indicates the running time of the solution in the graphic. The state space has six dimensions. Only the configuration is illustrated. {\bf Wheeled robot (bottom-right)}; The blue path indicates a \GLC solution with a quadratic angular rate penalty while the red path is a \GLC solution with a shortest path objective.}
\end{figure}

\subsection{Observations and Discussion}
The \RRTs algorithm was only tested on the first example since a steering function is unavailable for the remaining four.
The \SST algorithm was tested in all examples but the last since the implementation in \cite{BBekris2015} only supports min-time objectives.
We observe the run-time vs. objective-cost curves for the \GLC method is several orders of magnitude faster than the \SST algorithm. 
In the shortest path problem, the steering function for \RRTs is a line segment between two points.
In this case \RRTs outperforms both the \GLC and \SST methods and is the more appropriate algorithm.  
A more complex steering function can increase the running time of \RRTs by a considerable constant factor making it a less competitive option. 

%
In the 3D point robot example, we see the running time for the \GLC method to produce a (visually) good quality trajectory is roughly equal to the execution time.
This suggests six states is roughly the limit for real-time application without better heuristics. 
In the wheeled robot example a minimum time cost function was compared to a cost function which also penalized lateral acceleration.
The minimum time solution results in a path with abrupt changes in angular rate making it unsuitable for autonomous driving applications where passenger comfort is a consideration. 
Penalizing angular velocity resulted in a solution with more gradual changes in angular velocity.

Each data point in Figure \ref{fig:bench} represents the running time and solution cost of a complete evaluation of the \GLC method while \RRTs and \SST are incremental methods running until  interrupted. 
Each run of Algorithm \ref{Alg} operates on $\mathcal{U}_R$ for a fixed $R$. Since $\mathcal{U}_{R}\not\subset\mathcal{U}_{R+1}$ it is possible that an optimal signal in $\mathcal{U}_R$ may be better than any signal in $\mathcal{U}_{R+1}$ for some $R$ which explains the non-monotonic convergence observed. 
\section{\label{sec:Justification}Analysis of the GLC Condition}

Section \ref{sec:topology} begins by equipping $\mathcal{U}$ and $\mathcal{X}_{x_0}$ with metrics in order to discuss continuity of $\varphi_{x_0}$ and $J_{x_0}$. 
Using these metrics we can also discuss in what sense $\mathcal{U}_R$ approximates $\mathcal{U}$ in Section \ref{sec:discretization}. 
Finally, in Lemma \ref{lem:pruning} and Theorem \ref{thm:main}, we derive a bound on the gap between the cost of the solution output by Algorithm \ref{Alg} and the optimal cost for the problem and show that the gap converges to zero as the resolution is increased.

\subsection{Metrics on $\Xxo$ and $\U$, and Continuity of Relevant
Maps \label{sec:topology}}

The metrics introduced by Yershov and Lavalle~\cite{yershov2011sufficient} will be used.  
Recall, the signal space $\mathcal{U}$ was defined in (\ref{eq:signal_space}).
A metric on this space is given by 
\begin{equation}
d_{\mathcal{U}}(u_{1},u_{2})\coloneqq\int_{[0,\min(\tau(u_{1}),\tau(u_{2})]}\Vert u_{1}(t)-u_{2}(t)\Vert_{2} \, d\mu(t)+u_{max}|\tau(u_{1})-\tau(u_{2})|.\label{eq:du}
\end{equation}
The family of trajectory spaces are now more precisely defined as 
\begin{equation}
\mathcal{X}_{x_0} \coloneqq \bigcup_{\tau>0}\left\{ x:[0,\tau]\rightarrow \mathbb{R}^n: \: x(0)=x_{0}, \, \left\Vert \frac{x(t_{1})-x(t_{2})}{|t_{1}-t_{2}|}\right\Vert _{2}\leq M\:\, \forall t_{1,2}\in[0,\tau]\right\} ,\label{eq:traj_space}
\end{equation}
where $M$ is that of assumption A-2. This set is equipped with the metric 
\begin{equation}
d_{\mathcal{X}}(x_{1},x_{2})\equiv\max_{t\in\left[0,\min\{\tau(x_{1}),\tau(x_{2})\}\right]}\left\{ \left\Vert x_{1}(t)-x_{2}(t)\right\Vert \right\} +M|\tau(x_{1})-\tau(x_{2})|.\label{eq:traj_metric}
\end{equation}

Several known continuity properties of $\varphi_{x_0}$ are reviewed here. Recall (e.g.~\cite[pg. 95]{khalil1996nonlinear}), that the distance between solutions to (\ref{eq:dynamics}) with initial conditions $x_0$ and $z_0$ is bounded by
\begin{equation}\label{lem:cont_ic}
\left\Vert [\varphi_{x_{0}}(u)](t)-[\varphi_{z_{0}}(u)](t)\right\Vert _{2}\leq\Vert x_{0}-z_{0}\Vert_{2}e^{L_{f}t}.
\end{equation}  
In addition to continuous dependence on the initial condition parameter, the map $\varphi_{x_{0}}$ is also continuous from $\mathcal{U}$ into $\mathcal{X}_{x_{0}}$ (cf.~\cite[Theorem 1]{yershov2011sufficient}). It is a  useful observation then that $\mathcal{X}_\mathrm{feas}$ is open when assumption A-1 is satisfied (cf.~\cite[Theorem 2]{yershov2011sufficient}) since it follows directly from the definition of continuity that $\mathcal{U}_\mathrm{feas}$ and $\mathcal{U}_\mathrm{goal}$ are open subsets of $\mathcal{U}$; recall that they are defined as the preimage of $\mathcal{X}_\mathrm{feas}$ and $\mathcal{X}_\mathrm{goal}$ under $\varphi_{x_\mathrm{ic}}$.

Similar observations for the cost functional are developed below.

\begin{lemma} \label{lem:cost_cont}
$J_{x_0}:\mathcal{U}\rightarrow\mathbb{R}$ is continuous for any $x_0\in \mathbb{R}^n$.
\end{lemma}
\begin{proof}
Let $u_{1},u_{2}\in\mbox{\ensuremath{\mathcal{U}}}$ and without
loss of generality, let $\tau(u_{1})\leq\tau(u_{2})$. Denote
trajectories $\varphi_{x_{0}}(u_{1})$ and $\varphi_{x_{0}}(u_{2})$ by
$x_{1}$ and $x_{2}$ respectively. The associated difference in cost is 
\begin{equation}
\begin{array}{rl}
\left|J_{x_0}(u_{1})-J_{x_0}(u_{2})\right|= & \left|\int_{\left[0,\tau(u_{1})\right]}g\left(x_{1}(t),u_{1}(t)\right)\right.-g\left(x_{2}(t),u_{2}(t)\right)\, d\mu(t)\\
 & -\left.\int_{\left[\tau(u_{1}),\tau(u_{2})\right]}g\left(x_{2}(t),u_{2}(t)\right)\, d\mu(t)\right|.
\end{array}
\end{equation}
Using the Lipschitz constant of $g$ (cf. A-3), the definition of $d_{\mathcal{U} }$ in (\ref{eq:du}), and the definition of $d_{\mathcal{X} }$ in (\ref{eq:traj_metric}) the difference is bounded as follows: 
\begin{equation}
\begin{array}{rcl}
\left|J_{x_0}(u_{1})-J_{x_0}(u_{2})\right| & \leq & \left|\int_{\left[0,\tau(u_{1})\right]} L_{g}\left\Vert x_{1}(t)-x_{2}(t)\right\Vert _{2}\right.+L_{g}\left\Vert u_{1}(t)-u_{2}(t)\right\Vert _{2}\, d\mu(t) \\
 &  & -\left.\int_{\left[\tau(u_{1}),\tau(u_{2})\right]}g\left(x_{2}(t),u_{2}(t)\right)\, d\mu(t)\right|\\
 & \leq & L_{g}\tau(u_{1})\left\Vert x_{1}-x_{2}\right\Vert _{L_{\infty}\left[0,\tau(u_{1})\right]}+L_{g}\left\Vert u_{1}-u_{2}\right\Vert _{L_{1}\left[0,\tau(u_{1})\right]}\\
 &  & +\int_{\left[\tau(u_{1}),\tau(u_{2})\right]}\left|g\left(x_{2}(t),u_{2}(t)\right)\right|\, d\mu(t)\\
 & \leq & L_{g}\tau(u_{1}) d_{\mathcal{X}}(x_1,x_2)+ L_{g}d_{\mathcal{U}}(u_1,u_2)\\ 
& &  +\int_{\left[\tau(u_{1}),\tau(u_{2})\right]}\left|g\left(x_{2}(t),u_{2}(t)\right)\right|\, d\mu(t). 
\end{array}\label{eq:cost_diff}
\end{equation}
Since $x_2$ is continuous, $u_2$ is bounded, and $g$ is continuous, there exists a bound $G$ on $g(x_2(t),u_2(t))$ for $t\in [\tau(u_1),\tau(u_2)]$. Thus, the difference in cost is further bounded by  

\begin{equation}
\begin{array}{rcl}
\left|J_{x_0}(u_{1})-J_{x_0}(u_{2})\right| & \leq & L_{g}\tau(u_{1}) d_{\mathcal{X}}(x_1,x_2) + L_{g}d_{\mathcal{U}}(u_1,u_2) +G|\tau(u_2)-\tau(u_1)|. 
\end{array}
\label{eq:cost_diff2}
\end{equation}
Sine $\varphi_{x_0}$ is continuous, for $d_{\mathcal{U}}(u_1,u_2)$ sufficiently small the resulting trajectories will satisfy $L_{g}\tau(u_{1}) d_{\mathcal{X}}(x_1,x_2)<\varepsilon/3$. 
Additionally, for $u_1,u_2$ satisfying $d_{\mathcal{U}}(u_1,u_2)<\frac{\varepsilon u_{max}}{3G}$ and $d_{\mathcal{U}}(u_1,u_2)<\frac{\varepsilon}{3L_g}$ we have $G|\tau(u_2)-\tau(u_1)|<\varepsilon/3$  and $L_{g}d_{\mathcal{U}}(u_1,u_2)<\varepsilon/3$. For such a selection of $u_1,u_2$, we have $|J_{x_0}(u_1)-J_{x_0}(u_2)|<\varepsilon$. Thus, $J_{x_0}$ is continuous.
\qed
\end{proof} 
\begin{lemma}
\label{lem:cost_sensitivity}For any $u\in\mathcal{U}_\mathrm{feas}$ and
$x_{0},z_{0}\in\mathbb{R}^{n}$, 
\begin{equation}
|J_{x_{0}}(u)-J_{z_{0}}(u)|\leq\Vert x_{0}-z_{0}\Vert_{2}\cdot\frac{L_{g}}{L_{f}}\left(e^{L_{f}\tau(u)}-1\right)\label{eq:cost_sensitivity}
\end{equation}
\end{lemma}
\begin{proof}
The difference is bounded using the Lipschitz continuity of $g$. This is further bounded using (\ref{lem:cont_ic}). Denoting $x(t)=\varphi_{x_{0}}(u)$ and $z(t)=\varphi_{z_{0}}(u)$,
\begin{equation}
\begin{array}{rcl}
|J_{x_{0}}(u)-J_{z_{0}}(u)| & = & \left|\int_{[0,\tau(u)]}g(x(t),u(t)) -g(z(t),u(t))\, d\mu(t) \right|\\
 & \leq & \int_{[0,\tau(u)]}\left|g(x(t),u(t))\right. \left.-g(z(t),u(t))\right|\, d\mu(t)\\
 & \leq & \int_{[0,\tau(u)]}L_{g}\left\Vert x(t)-z(t)\right\Vert _{2}\, d\mu(t)\\
 & \leq & \int_{[0,\tau(u)]}\Vert x_{0}-z_{0}\Vert_{2} L_{g}e^{L_{f}t}\, d\mu(t)\\
 & = & \Vert x_{0}-z_{0}\Vert_{2}\frac{L_{g}}{Lf}\left(e^{L_{f}\tau(u)}-1\right).
\end{array}
\end{equation}
 \qed
\end{proof}

\subsection{Properties of the Approximation of $\mathcal{U}$ by $\mathcal{U}_{R}$ \label{sec:discretization}}


Note that Lemma \ref{lem:density} is not a statement about the dispersion of $\mathcal{U}_R$ in $\mathcal{U}$ which does not actually converge. For numerical approximations of function spaces the weaker statement that $\limsup_{R\rightarrow \infty}\mathcal{U}_R$ is dense in $\mathcal{U}$ will be sufficient. Equivalently,
\begin{lemma}
\label{lem:density}For each $u\in\mathcal{U}$ and $\varepsilon>0$,
there exists $R^{*}>0$ such that for any $R>R^{*}$ there exists
$w\in\mathcal{U}_{R}$ such that $d_{\mathcal{U}}(u,w)<\varepsilon$.\end{lemma}
\begin{proof}

By Lusin's Theorem~\cite[pg. 41]{kolmogorov1961elements}, there exists
a continuous $\upsilon:[0,\tau(u)]\rightarrow\Omega$ such that 
\begin{equation}
\mu(\{t\in[0,\tau(u)]:\:\upsilon(t)\neq u(t)\})<\frac{\varepsilon}{3u_{max}}.
\end{equation}
The domain of $\upsilon$ is compact so $\upsilon$ is also uniformly continuous. Denote its modulus of continuity by $\delta(\epsilon)$,
\begin{equation}
|\sigma-\gamma|<\delta(\epsilon)\Rightarrow\Vert\upsilon(\sigma)-\upsilon(\gamma)\Vert_{2}<\epsilon.
\end{equation}

To construct an approximation of $\upsilon$ by $w\in\mathcal{U}_{R}$
choose $R$ sufficiently large so that (i) $h(R)/R>\tau(u)$,
(ii) there exists an integer $r$ such that $0<\tau(u)-r/R<1/R<\delta\left(\frac{\varepsilon}{6\tau(u)}\right)$,
and (iii) the dispersion of $\Omega_{R}$ in $\Omega$ is less than $\frac{\varepsilon}{6\tau(u)}$. 

Then for $t\in[(i-1)/R,i/R)$ $i=1,...,r$ there exists ${\rm v}_{i}\in\Omega_{R}$
such that $\Vert {\rm v}_{i}-\upsilon(t)\Vert_{2}<\frac{\varepsilon}{3\tau(u)}.$
Select $w\in\mathcal{U}_{R}$ which is equal to ${\rm v}_{i}$ on
each interval. Combining (i)-(iii), 
\begin{equation}
\begin{array}{rcl}
d_{\mathcal{U}}(\upsilon,w) & = & \int_{[0,r/R]}\Vert w(t)-\omega(t)\Vert_{2},\, d\mu(t)+u_{max}|r/R-\tau(\omega)|\\
 & < & \int_{[0,r/R]}\Vert w(t)-\omega(t)\Vert_{2},\, d\mu(t)+\frac{\varepsilon}{6}\\
 & < & \int_{[0,r/R]}\frac{\varepsilon}{3\tau(u)},\, d\mu(t)+\frac{\varepsilon}{6}\\
 & < & \frac{\varepsilon}{2}.
\end{array}
\end{equation}
Thus, by the triangle inequality
\begin{equation}
d_{\mathcal{U}}(u,w)\leq d_{\mathcal{U}}(u,\upsilon)+d(\upsilon,w)<\varepsilon.
\end{equation}
\qed
\end{proof}

We use $cl(\cdot)$ and $int(\cdot)$ to denote the closure and interior of subsets of $\mathcal{U}$.
\begin{lemma}
\label{lem:no_limit_points}For any $w\in cl\left(int\left(\mathcal{U}_\mathrm{goal}\right)\right)$
and $\varepsilon>0$ there exists $R^{*}>0$ such that for any $R>R^{*}$
\begin{equation}
\min_{u\in\mathcal{U}_{R}\cap\mathcal{U}_\mathrm{goal}}\left\{ \left|J_{x_\mathrm{ic}}(u)-J_{x_\mathrm{ic}}(w)\right|\right\} <\varepsilon.
\end{equation}
\end{lemma}
\begin{proof}
$\omega\in cl\left(int\left(\mathcal{U}_\mathrm{goal}\right)\right)$ implies
that for all $\delta>0$, $B_{\delta/2}(\omega)\cap int(\mathcal{U}_\mathrm{goal})\neq\emptyset$.
Then each $\upsilon\in int(\mathcal{U}_\mathrm{goal})$ has a neighbourhood
$B_{\rho}(\upsilon)\subset int(\mathcal{U}_\mathrm{goal})$ with $\rho<\frac{\delta}{2}$.
Take $\upsilon\in B_{\delta/2}(\omega)\cap int(\mathcal{U}_{goal.})$.
By Lemma \ref{lem:density}, for sufficiently large $R$, there exists $u\in\mathcal{U}_{R}$ such
that $d_{\mathcal{U}}(\upsilon,u)<\delta/2$ which implies $u\in B_{\rho}(\upsilon)\subset int(\mathcal{U}_\mathrm{goal})$.
Then $u\in\mathcal{U}_\mathrm{goal}$ and $d_{\mathcal{U}}(\omega,u)<d_{\mathcal{U}}(\omega,\upsilon)+d_{\mathcal{U}}(\upsilon,u)<\delta$. Now by the continuity of $J_{x_\mathrm{ic}}$ for $\delta$ sufficiently small, $d_{\mathcal{U}}(\omega,\upsilon)<\delta$ implies $|J_{x_\mathrm{ic}}(u)-J_{x_\mathrm{ic}}(w)| <\varepsilon$ from which the result follows.
\qed
\end{proof}
A sufficient condition for every $\omega\in\mathcal{U}_\mathrm{goal}$ to
be contained in the closure of the interior of $\mathcal{U}_\mathrm{goal}$ is that $\mathcal{U}_\mathrm{goal}$ be open which is the case when Assumption A-1 is satisfied. 
%

\subsection{Pruning $\mathcal{U}_{R}$ with the GLC Condition \label{sec:pruning}}

To describe trajectories remaining on the $\varepsilon$-interior of $X_\mathrm{free}$ at each instant and which terminate on the $\varepsilon$-interior of $X_\mathrm{goal}$ the following sets are defined,
\begin{equation}
\begin{array}{l}
\mathcal{X}_\mathrm{goal}^{\varepsilon}\coloneqq\left\{ x\in\mathcal{X}_\mathrm{goal}:\,\; B_{\varepsilon}(x(t))\subset X_\mathrm{free}\, \forall t\in[0,\tau(x)], \, B_{\varepsilon}(x(\tau(x))\subset X_\mathrm{goal}\right\} ,\\
\mathcal{U}_\mathrm{goal}^{\varepsilon}\coloneqq\left\{ u\in\mathcal{U}_\mathrm{goal}:\,\varphi_{x_\mathrm{ic}}(u)\in\mathcal{X}_\mathrm{goal}^{\varepsilon}\right\} ,
\end{array}\label{eq:interiors}
\end{equation}
and similarly for $c_{R}$,
\begin{equation}
c_{R}^{\varepsilon}\coloneqq\min_{u\in\mathcal{U}_{R}\cap\mathcal{U}_\mathrm{goal}^{\varepsilon}}\left\{ J_{x_\mathrm{ic}}(u)\right\} .\label{eq:cr_ep}
\end{equation}
 Since $\ensuremath{\mathcal{U}_\mathrm{goal}^{\varepsilon}\subset\mathcal{U}_\mathrm{goal}}$
we have that $0\leq c_{R}\leq c_{R}^{\varepsilon}$. 
\begin{lemma}
\label{lem:approx_equal_optimal}If $\lim_{R\rightarrow\infty}\epsilon(R)=0$
then $\lim_{R\rightarrow\infty}c_{R}^{\epsilon(R)}=c^{*}$.\end{lemma}
\begin{proof}
By the definition of $c^{*}$ in (\ref{eq:meaningful_problem}), for any
$\varepsilon>0$ there exists $\omega\in\mathcal{U}_\mathrm{goal}$ such
that $J_{x_\mathrm{ic}}(\omega)-\varepsilon/2<c^{*}$. Since
$\mathcal{U}_\mathrm{goal}$ is open and $\varphi_{x_\mathrm{ic}}$ is continuous there exists $\tilde{r}>0$ and $\rho>0$
such that $B_{\tilde{r}}(\omega)\subset\mathcal{U}_\mathrm{goal}$ and $\varphi_{x_\mathrm{ic}}(B_{\tilde{r}}(\omega))\subset B_{\rho}(\varphi_{x_\mathrm{ic}}(\omega))$.
Thus, $\omega\in\mathcal{U}_\mathrm{goal}^{\tilde{r}}$. Similarly, there exists a positive $r<\tilde{r}$ such that $\varphi_{x_\mathrm{ic}}(B_{r}(\omega))\subset B_{\rho/2}(\varphi_{x_\mathrm{ic}}(\omega))$
and $B_{r}(\omega)\subset\mathcal{U}_\mathrm{goal}^{\rho/2}$. From the continuity
of $J$ in Lemma \ref{lem:cost_cont} there also exists a positive
$\delta<r$ such that for any signal $\upsilon$ with $d_{\mathcal{U}}(\upsilon,\omega)<\delta$
we have $|J_{x_\mathrm{ic}}(\omega)-J_{x_\mathrm{ic}}(\upsilon)|<\varepsilon/2$.

Next, choose $R^{*}$ to be sufficiently large such that $R>R^*$ implies $\mbox{\ensuremath{\epsilon}}(R)<\rho/2$
and $B_{\delta}(\omega)\cap\mathcal{U}_{R}\neq\emptyset$. Such a resolution $R^{*}$ exists by
Lemma \ref{lem:no_limit_points} and the assumption $\lim_{R\rightarrow\infty}\epsilon(R)=0$.
Now choose $u\in B_{\delta}(\omega)\cap\mathcal{U}_{R}$.
Then $|J_{x_\mathrm{ic}}(u)-J_{x_\mathrm{ic}}(\omega)|<\varepsilon/2$ and $u\in\mathcal{U}_\mathrm{goal}^{\rho/2}\subset\mathcal{U}_\mathrm{goal}^{\epsilon(R)}$.
Then by definition of $c_{R}^{\epsilon(R)}$, $u\in\mathcal{U}_\mathrm{goal}^{\epsilon(R)}$
implies $c_{R}^{\epsilon(R)}\leq J_{x_\mathrm{ic}}(u)$. Finally, by triangle
inequality, 
\begin{equation}
|J_{x_\mathrm{ic}}(u)-c^{*}|<|J_{x_\mathrm{ic}}(u)-J_{x_\mathrm{ic}}(\omega)|+|J_{x_\mathrm{ic}}(\omega)-c^{*}|<\varepsilon.
\end{equation}
Rearranging the expression yields $J_{x_\mathrm{ic}}(u)<c^{*}+\varepsilon$
and thus, $c_{R}^{\epsilon(R)}<c^{*}+\varepsilon$. The result follows
since the choice of $\varepsilon$ is arbitrary.
\qed
\end{proof}

To simplify the notation in what follows, a concatenation operation on elements of $\mathcal{U}$ is defined.
For $u_{1},u_{2}\in\mathcal{U}$, their concatenation $u_{1}u_{2}$
is defined by 
\begin{equation}
[u_{1}u_{2}](t)\coloneqq\left\{ \begin{array}{c}
u_{1}(t),\, t\in[0,\tau(u_{1}))\\
u_{2}(t-\tau(u_{1})),\, t\in[\tau(u_{1}),\tau(u_{1})+\tau(u_{2})]
\end{array}\right..\label{eq:concatenation}
\end{equation}
The concatenation operation will be useful together
with the following equalities which are readily verified, 
\begin{equation}\label{eq:cost_homo}
J_{x_{0}}(u_{1}u_{2})=J_{x_{0}}(u_{1})+J_{[\varphi_{x_0}(u_1)](\tau(u_1))}(u_{2})
\end{equation}
\begin{equation}
\ensuremath{\varphi_{x_{0}}(u_{1}u_{2})=\varphi_{x_{0}}(u_{1})\varphi_{[\varphi_{x_{0}}(u_1)](\tau(u))}(u_{2})}.\label{eq:x_concat}
\end{equation}
The concatenation operation on $\mathcal{X}_{x_{0}}$ in (\ref{eq:x_concat})
is defined in the same way as in (\ref{eq:concatenation}). 
\begin{lemma}
\label{lem:pruning}
Let $\delta_0=\frac{\sqrt{n}}{\eta(R)}e^{\frac{L_{f}h(R)}{R}}$.
For $\varepsilon\geq \delta_0$ and $u_{i},u_{j}\in\mathcal{U}_{R}\cap\mathcal{U}_\mathrm{feas}$, if
$u_{i}\prec_{R}u_{j}$, then for each descendant of $u_{j}$ in $\mathcal{U}_\mathrm{goal}^{\varepsilon-\delta_0}$
with cost $c_{j}$, there exists a descendant of $u_{i}$ in $\mathcal{U}_\mathrm{goal}$
with cost $c_{i}\leq c_{j}$.
\end{lemma}
\begin{proof}
Suppose there is a $w\in\mathcal{U}_{R}$ such that $u_{j}w\in\mathcal{U}_{R}$
is a descendant of $u_{j}$ and $u_{j}w\in\mathcal{U}_\mathrm{goal}^{\varepsilon}.$
 Since $u_{i}\overset{R}{\sim}u_{j}$,
\begin{equation}
\left\Vert[\varphi_{x_\mathrm{ic}}(u_{i}w)](\tau(u_{i}))-[\varphi_{x_\mathrm{ic}}(u_{j}w)](\tau(u_{j}))\right\Vert_2\leq \frac{\sqrt{n}}{\eta(R)}.
\end{equation}
Then, by equation (\ref{lem:cont_ic}), for all $t\in[0,\tau(w)]\subset[0,h(R)/R]$, 
\begin{equation}
\Vert[\varphi_{x_\mathrm{ic}}(u_{i}w)](t+\tau(u_{i}))-[\varphi_{x_\mathrm{ic}}(u_{j}w)](t+\tau(u_{j}))\Vert\leq\frac{\sqrt{n}}{\eta(R)}e^{\frac{L_{f}h(R)}{R}}.\label{eq:-8}
\end{equation}
Then for $\delta_0=\frac{\sqrt{n}}{\eta(R)}e^{\frac{L_{f}h(R)}{R}}$
we have $u_{i} w \in \mathcal{U}_\mathrm{goal}^{\varepsilon-\delta_0}.$ 
As for the cost, from equation (\ref{eq:cost_homo}) 
\begin{equation}
\begin{array}{rcl}
J_{x_\mathrm{ic}}(u_{i} w) & = & J_{x_\mathrm{ic}}(u_{i})+J_{[\varphi_{x_\mathrm{ic}}(u_{i})](\tau(u_{i}))}(w)\\
 & \leq & J_{x_\mathrm{ic}}(u_{j})+\frac{\sqrt{n}}{\eta(R)}\frac{L_{g}}{L_{f}}\left(e^{\frac{L_{f}h(R)}{R}}-1\right) +J_{[\varphi_{x_\mathrm{ic}}(u_{j})](\tau(u_{j}))}(w)\\
 & \leq & J_{x_\mathrm{ic}}(u_{j})+J_{[\varphi_{x_\mathrm{ic}}(u_{j})](\tau(u_{j}))}(w)\\
 & \leq  & J_{x_\mathrm{ic}}(u_{j}w)\\
\end{array}
\end{equation}
 The first step applies the conditions for $u_{i}\prec_{R}u_{j}$.
The second step combines Lemma \ref{lem:cost_sensitivity} and $u_{i}\overset{N}{\sim}u_{j}$. 
\qed
\end{proof}
In the above Lemma, the quantity $\delta_0$ was constructed based on the radius of the partition of $X_\mathrm{free}$ and the sensitivity of solutions to initial conditions. 
Let $\delta_k$ be defined as a finite sum of related quantities, $\delta_k\coloneqq\sum_{i=1}^{k}\frac{\sqrt{n}}{\eta(R)}e^{\frac{L_{f}(h(R)-i)}{R}}$. The following bound will be useful to state our main result in the next Theorem,
\begin{equation}
\delta_{h(R)}=\sum_{i=1}^{h(R)}\frac{\sqrt{n}}{\eta(R)}e^{\frac{L_{f}(h(R)-i)}{R}} \leq\frac{R\sqrt{n}}{L_{f}\eta(R)}\left(e^{L_{f}h(R)/R}-1\right).\label{eq:inequality}
\end{equation}
\begin{theorem}\label{thm:main}
	Let $\epsilon(R)=\frac{R\sqrt{n}}{L_{f}\eta(R)}\left(e^{\frac{L_{f}h(R)}{R}}-1\right)$. Algorithm \ref{Alg} terminates in finite time and returns a solution with cost less than or equal to $c^{\epsilon(R)}_R$.
\end{theorem}
\begin{proof}
	The queue is a subset of $\mathcal{U}_R\cup \{Id_\mathcal{U} \}$ and at line 3 in each iteration a lowest cost signal $u$ is removed from the queue. 
	In line 13, only children of the current signal $u$ are added to the queue. Since $\mathcal{U}_R$ is organized as a tree and has no cycles, any signal $u$ will enter the queue at most once.
	%
	Therefore the queue must be empty after a finite number of iterations so the algorithm terminates.

	Next, consider as a point of contradiction the hypothesis that the output has cost greater than $c_R^{\epsilon(R)}$. 
	Then it is necessary that $c_R^{\epsilon(R)}<\infty$ and by the definition of  $c_R^{\epsilon(R)}$  in (\ref{eq:cr_ep}), it is also necessary that $\mathcal{U}_{R}\cap \mathcal{U}_\mathrm{goal}^{\epsilon(R)}$ is non-empty. 
	
	Choose $u^*\in \mathcal{U}_{R}\cap\mathcal{U}_\mathrm{goal}^{\epsilon(R)}$ with cost $J_{x_\mathrm{ic}}(u^*)=c_R^{\epsilon(R)}$. 
	It follows from the hypothesis that $u^*$ does not enter the queue.
	Otherwise, by (\ref{eq:queue}) it would be evaluated before any signal of cost greater than $c_R^{\epsilon(R)}$.
	If $u^*$ does not enter the queue, then a signal $u_0$ must at some iteration be present in $\Sigma$ which prunes an ancestor $a_0$ of $u^*$ ($u_0 \prec_R a_0$ in line 9). 
	This ancestor must satisfy $\texttt{depth}(a_0)>0$ since the ancestor with depth $0$ is $Id_\U$ which enters queue in line 1. 
	By Lemma \ref{lem:pruning}, $u_0$ has a descendant of the form $u_0d_0\in \Ugoal^{\epsilon(R)-\delta_1}$ and $J_{x_\mathrm{ic}}(u_0d_0)\leq c_R^{\epsilon(R)}$. 
	Additionally, $\mathtt{depth}(a_0)>0$ implies $\mathtt{depth}(d_0)\leq h(R)-1$.

	Having pruned $u^*$ in line 9, the signal $u_0\in\Sigma$, or a sibling which prunes $u_0$ (and by transitivity, prunes $u^*$) must at some point be present in the queue (cf. line 12-13). 
	Of these two, denote the one that ends up in the queue by $\tilde{u}_0$.
	Since $\tilde{u}_0$ is at some point present in the queue and $\tilde{u}_0d_0\in\Ugoal^{\epsilon(R)-\delta_1}$, a signal $u_1\in\Sigma$ must prune an ancestor $a_1$ of $\tilde{u}_0d_0$ ($u_1 \prec_R a_1$ in line 9). 
	Since $\tilde{u}_0$ is at some point present in the queue, the ancestor $a_1$ of $\tilde{u}_0d_0$, must have greater depth than $\tilde{u}_0$. 
	By Lemma \ref{lem:pruning}, $u_1$ has a descendant of the form $u_1d_1\in \Ugoal^{\epsilon(R)-\delta_2}$ and $J_{x_\mathrm{ic}}(u_1d_1)\leq c_R^{\epsilon(R)}$. 
	Additionally, $\mathtt{depth}(a_1)> \tilde{u}_0$ implies $\mathtt{depth}(d_1)\leq \mathtt{depth}(d_0)-1 \leq h(R)-2$.

	Continuing this line of deduction leads to the observation that a signal $u_{h(R)-1}$, with a descendant of the form $u_{h(R)-1}d_{h(R)-1}\in \Ugoal^{\epsilon(R)-\delta_{h(R)-1}}$ and $J_{x_\mathrm{ic}}(u_{h(R)-1}d_{h(R)-1})\leq c_R^{\epsilon(R)}$, will be present in the queue; and $\mathtt{depth}(d_{h(R)-1})\leq 1$. 
	Since $u_{h(R)-1}$ is at some point present in the queue, a signal $u_{h(R)}\in\Sigma$ must prune an ancestor $a_{h(R)}$ of $u_{h(R)-1}d_{h(R)-1}$ ($u_{h(R)} \prec_R a_{h(R)}$ in line 9).
	Since $u_{h(R)-1}$ is at some point present in the queue, the ancestor $a_{h(R)}$ of $u_{h(R)-1}d_{h(R)-1}$, must have greater depth than $u_{h(R)-1}$, and therefore, is equal to $u_{h(R)-1}d_{h(R)-1}$. 
	Thus, $u_{h(R)}\in \Ugoal^{\epsilon(R)-\delta_{h(R)}}$ (note that $\epsilon(R)-\delta_{h(R)})\geq0$ by \eqref{eq:inequality}) and $J(u_{h(R)})\leq c_R^{\epsilon(R)}$. 
	Then $u_{h(R)}$ or a sibling which prunes $u_{h(R)}$ will be added to the queue; a contradiction of the hypothesis since this signal will be removed from the queue and the algorithm will terminate, returning this signal in line 7. 
	\qed
\end{proof}

The choice of $\epsilon(R)$ in Theorem \ref{thm:main} converges to zero by  (\ref{eq:partition_scaling}).
Then by Lemma \ref{lem:approx_equal_optimal} we have $c_{R}^{\epsilon(R)}\rightarrow c^{*}$. 
An immediate corollary is that the \GLC method is resolution complete which is the main contribution of this work stated in Theorem \ref{thm:informal}.

\section{Conclusion\label{sec:Conclusions}}
In this paper we described a simple grid-based approximation of the optimal kinodynamic motion planning problem and developed the appropriate generalization of label correcting methods to efficiently search the approximation.
The advantage of the \GLC method is that it does not require a point-to-point local planning subroutine. 
Moreover, numerical experiments demonstrate that the \GLC method is considerably faster the related \SST algorithm, and is more broadly applicable than the \RRTs algorithm. 

%

The focus of the paper, and its main contribution was the theoretical investigation showing that the cost of the feasible solutions returned by the \GLC method converge to the optimal cost for the problem.
From a practical point of view the proposed algorithm is easy to implement and may be used directly for motion planning and trajectory optimization.
Future investigations will include convergence rate analysis, improved partitioning schemes over the hypercube grid we presented, and the construction of admissible heuristics for kinodynamic motion planning to guide the search.
\pagebreak
\section*{Acknowledgements}
This research was funded in part by the Israeli Ministry of Defense. 
We are also grateful to our colleagues Michal \v C\'ap and Dmitry Yershov for their insightful comments. 
\bibliographystyle{splncs}
\bibliography{references}

\begin{thebibliography}{10}

\bibitem{PRM}
Kavraki, L.E., Svestka, P., Latombe, J.C., Overmars, M.H.:
\newblock {Probabilistic Roadmaps for Path Planning in High-Dimensional
  Configuration Spaces}.
\newblock {IEEE Transactions on Robotics and Automation} \textbf{12}(4) (1996)
  566--580

\bibitem{karaman2011sampling}
Karaman, S., Frazzoli, E.:
\newblock {Sampling-Based Algorithms for Optimal Motion Planning}.
\newblock The International Journal of Robotics Research \textbf{30}(7) (2011)
  846--894

\bibitem{EST_Journal}
Hsu, D., Kindel, R., Latombe, J.C., Rock, S.:
\newblock Randomized kinodynamic motion planning with moving obstacles.
\newblock The International Journal of Robotics Research \textbf{21}(3) (2002)
  233--255

\bibitem{RRT_Journal}
LaValle, S.M., Kuffner, J.J.:
\newblock Randomized kinodynamic planning.
\newblock The International Journal of Robotics Research \textbf{20}(5) (2001)
  378--400

\bibitem{karaman2010optimal}
Karaman, S., Frazzoli, E.:
\newblock {Optimal Kinodynamic Motion Planning Using Incremental Sampling-Based
  Methods}.
\newblock In: IEEE Conference on Decision and Control. (2010)  7681--7687

\bibitem{perez2012lqr}
Perez, A., Platt~Jr, R., Konidaris, G., Kaelbling, L., Lozano-Perez, T.:
\newblock {LQR-RRT*: Optimal Sampling-Based Motion Planning with Automatically
  Derived Extension Heuristics}.
\newblock In: International Conference on Robotics and Automation, IEEE (2012)
  2537--2542

\bibitem{xie2015toward}
Xie, C., van~den Berg, J., Patil, S., Abbeel, P.:
\newblock {Toward Asymptotically Optimal Motion Planning for Kinodynamic
  Systems Using a Two-Point Boundary Value Problem Solver}.
\newblock In: {International Conference on Robotics and Automation}, IEEE
  (2015)  4187--4194

\bibitem{dijkstra1959note}
Dijkstra, E.W.:
\newblock {A Note on Two Problems in Connexion with Graphs}.
\newblock Numerische mathematik \textbf{1}(1) (1959)  269--271

\bibitem{dolgov2010path}
Dolgov, D., Thrun, S., Montemerlo, M., Diebel, J.:
\newblock Path planning for autonomous vehicles in unknown semi-structured
  environments.
\newblock The International Journal of Robotics Research \textbf{29}(5) (2010)
  485--501

\bibitem{Li2016Asymptotically-}
Li, Y., Littlefield, Z., Bekris, K.E.:
\newblock {Sparse Methods for Efficient Asymptotically Optimal Kinodynamic
  Planning}.
\newblock In: {Algorithmic Foundations of Robotics XI}.
\newblock Springer (2015)  263--282

\bibitem{coddington1955theory}
Coddington, E.A., Levinson, N.:
\newblock {Theory of Ordinary Differential Equations}.
\newblock {McGraw-Hill Education} (1955)

\bibitem{hart1968formal}
Hart, P.E., Nilsson, N.J., Raphael, B.:
\newblock {A Formal Basis for the Heuristic Determination of Minimum Cost
  Paths}.
\newblock Systems Science and Cybernetics \textbf{4}(2) (1968)  100--107

\bibitem{rrt_implementation}
Karaman, S.:
\newblock {RRT* Library}.
\newblock \url{http://karaman.mit.edu/software.html}

\bibitem{BBekris2015}
Li, Y., Littlefield, Z., Bekris, K.E.:
\newblock Sparse {RRT} package.
\newblock \url{https://bitbucket.org/pracsys/sparse_rrt/} Accessed: Jan. 2016.

\bibitem{spong1995swing}
Spong, M.W.:
\newblock {The Swing Up Control Problem for the Acrobot}.
\newblock Control Systems, IEEE \textbf{15}(1) (1995)  49--55

\bibitem{yershov2011sufficient}
Yershov, D.S., LaValle, S.M.:
\newblock {Sufficient Conditions for the Existence of Resolution Complete
  Planning Algorithms}.
\newblock In: Algorithmic Foundations of Robotics IX.
\newblock Springer (2011)  303--320

\bibitem{khalil1996nonlinear}
Khalil, H.K., Grizzle, J.:
\newblock {Nonlinear Systems}. Volume~3.
\newblock Prentice hall New Jersey (1996)

\bibitem{kolmogorov1961elements}
Kolmogorov, A.N., Fomin, S.V.:
\newblock {Elements of the Theory of Functions and Functional Analysis. Vol. 2,
  Measure. The Lebesgue Integral. Hilbert Spaces}.
\newblock Graylock Press (1961)

\end{thebibliography}

\end{document}